%% file: main.tex
\documentclass[10pt,twocolumn,letterpaper]{article}

\usepackage[T1]{fontenc}
\usepackage[pagenumbers]{cvpr}
\input{preamble}

\definecolor{cvprblue}{rgb}{0.21,0.49,0.74}
\usepackage[pagebackref,breaklinks,colorlinks,allcolors=cvprblue]{hyperref}

\def\confName{CVPR}
\def\confYear{2027}

\title{PrimFlow: Few-Step Trajectory Distillation with Adaptive Temporal Function Parameterization}

\author{%
Hongyi Fang\textsuperscript{1}, Chuwen Xie\textsuperscript{1}, Chenggong Hu\textsuperscript{2},
Benjia Zhou\textsuperscript{1}, Yu-Xuan Qiu\textsuperscript{1}, Rui Mao\textsuperscript{3}\\[2pt]
{\small\textsuperscript{1}Beijing Institute of Technology\quad\textsuperscript{2}Zhejiang University\quad\textsuperscript{3}Shenzhen University}%
}

\begin{document}
\twocolumn[{
\renewcommand\twocolumn[1][]{#1}
\maketitle
\vspace{-1.2em}
\begin{center}
    \includegraphics[width=\textwidth]{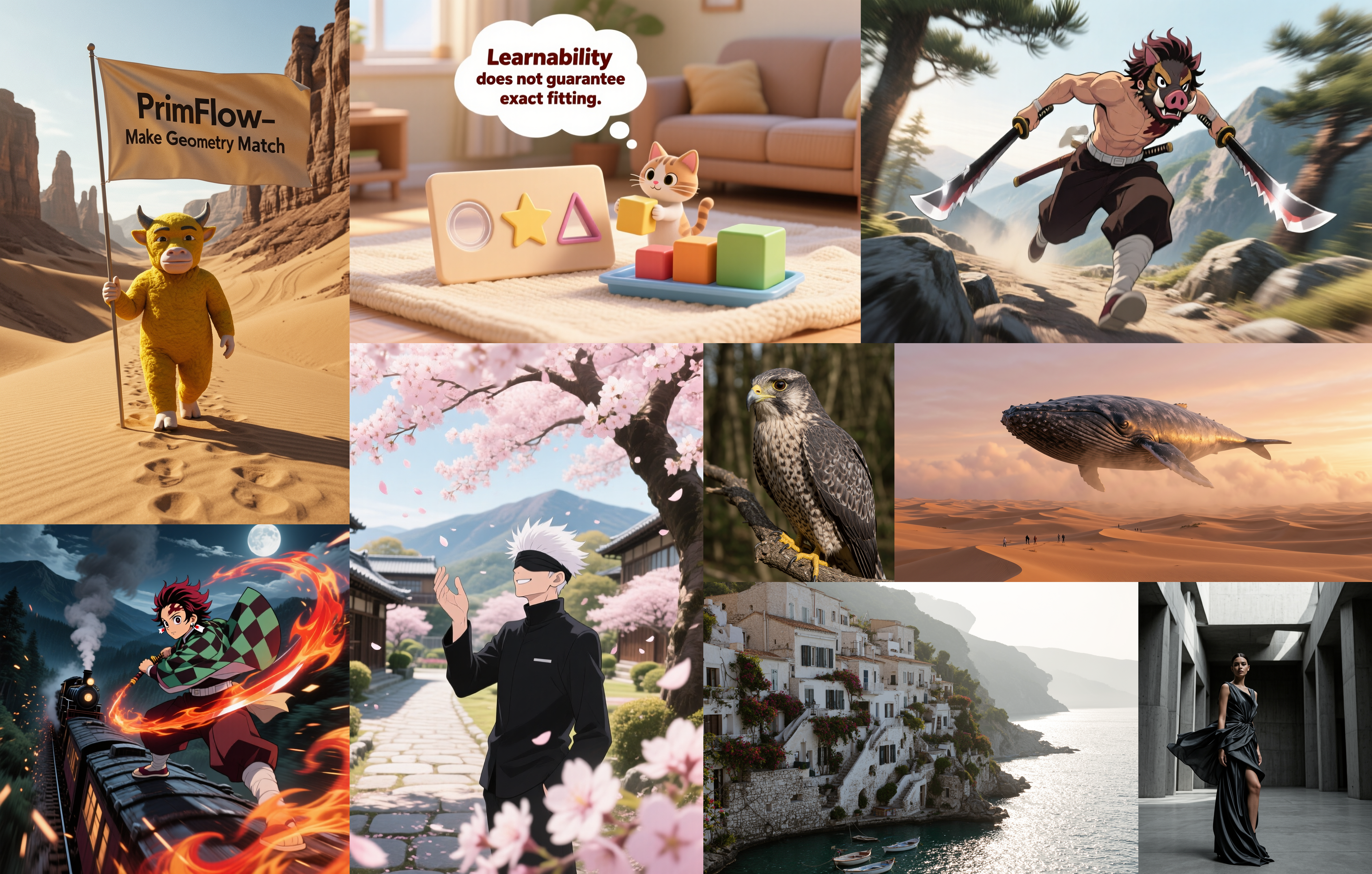}
\end{center}
\vspace{-0.8em}
\captionsetup{hypcap=false}
\captionof{figure}{Qualitative samples generated by PrimFlow using different backbone models and numbers of function evaluations (NFEs).}
\label{fig:teaser}
\captionsetup{hypcap=true}
\vspace{0.4em}
}]
\input{sec/0_abstract}

\clearpage
\section*{LLM Usage Statement}
OpenAI Codex and its underlying generative AI models were used as assistive
tools for implementing and running experiments, preparing plotting code and
figures, and polishing the manuscript's language and presentation. The core
research ideas, method design, scientific judgments, interpretation of
results, and final manuscript content were provided, verified, and approved
by the human authors. The authors take full responsibility for the accuracy,
integrity, and reproducibility of the submitted work.

{\small
\bibliographystyle{ieeenat_fullname}
\bibliography{main}
}
\end{document}

%% file: preamble.tex
\usepackage{amsmath,amssymb,amsfonts,amsthm,mathtools}
\usepackage{booktabs,multirow,tabularx,array}
\usepackage{graphicx}
\usepackage{xcolor}
\usepackage{microtype}
\usepackage{xspace}
\usepackage{enumitem}

\newcommand{\method}{\textsc{PrimFlow}\xspace}

\definecolor{placeholderbg}{RGB}{245,247,250}
\definecolor{placeholderline}{RGB}{160,171,185}

\setlist[itemize]{leftmargin=*,topsep=2pt,itemsep=1pt,parsep=0pt}
\setlist[enumerate]{leftmargin=*,topsep=2pt,itemsep=1pt,parsep=0pt}

%% file: sec/0_abstract.tex
\begin{abstract}
How to compress the hundreds of network evaluations required by diffusion and flow models to generate content into only a few steps while maintaining generation quality is an important concern in industry. Trajectory distillation enables a few-step model to fit the ODE trajectory obtained by integrating the teacher's multistep dynamics through a single backbone computation, providing an effective approach to few-step generation. Existing methods typically use temporal function parameterization (TFP), in which the backbone predicts parameters of a predefined temporal trajectory function. These designs generally express only trajectories with particular structures; even fully optimizing their learnable parameters does not guarantee an accurate fit to an arbitrary teacher trajectory. We introduce PrimFlow, which extracts shared latent features with one backbone forward pass and uses a lightweight auxiliary network conditioned on query time to learn an adaptive temporal function, relaxing this restriction. Text-to-image experiments on FLUX.2-Klein-base-9B and Qwen-Image-20B show better overall performance across multiple evaluation metrics, backbones, and 2/4 backbone-evaluation settings.
\end{abstract}

\section{Introduction}
\label{sec:introduction}
Diffusion and flow models generate high-quality, diverse multimedia content, including images, videos, and audio~\cite{DBLP:journals/corr/abs-2508-02324,DBLP:journals/corr/abs-2503-20314,DBLP:conf/cvpr/ChengIH0SM25}. Their standard generation procedures, however, often require hundreds of network evaluations, incurring substantial computation and inference latency that limit practical adoption~\cite{DBLP:conf/iclr/SalimansH22}. Efficiently and reliably distilling pretrained multistep generative models into few-step models has therefore become an active research direction~\cite{DBLP:conf/nips/SabourFK25,DBLP:conf/cvpr/GengLWSKH26,DBLP:conf/cvpr/FanQWWLRLGY26}.

Compared with distribution distillation methods that use distribution-matching objectives~\cite{DBLP:conf/cvpr/YinG0SDFP24,DBLP:conf/eccv/SauerLBR24}, trajectory distillation directly uses supervision from the teacher along its generation trajectories. It typically requires neither variational score distillation (VSD) nor adversarial distribution-matching losses, avoiding their additional optimization complexity. How can a student construct a trajectory that aligns with the teacher's integrated multistep trajectory using only one backbone forward pass? Existing methods commonly address this through \emph{temporal function parameterization} (TFP)~\cite{DBLP:journals/corr/abs-2602-09014,DBLP:conf/iclr/ChenZTGWB26,DBLP:journals/corr/abs-2607-26004}. A temporal function with a specific structure is designed in advance, and some of its parameters are predicted by the generative model in one backbone evaluation. Different parameter values produce different trajectories within this function family. The model learns these parameters to fit the teacher's multistep trajectory as closely as possible. Its expressive power, however, remains constrained by the predefined temporal structure. If the family cannot approximate the teacher trajectory arbitrarily well in the uniform norm, an irreducible trajectory approximation error remains even after fully optimizing all learnable parameters. Relaxing the representation restrictions imposed by predefined TFP structures is thus a key question for improving few-step trajectory distillation.

We introduce PrimFlow, an adaptive output-side TFP for pretrained single-time backbones. A single evaluation of the pretrained backbone supplies the source features. PrimFlow conditions its readout on the query time to construct a velocity antiderivative anchored at the source, so one learned displacement supplies both local velocity supervision and a direct interval update. A finite-difference approximation to its query-time derivative supplies the student velocity for learning from the teacher's instantaneous velocities. Multiple query times associated with the same source state and condition share the backbone features. Consequently, the temporal queries required by finite differences operate only on the lightweight auxiliary network and require no repeated backbone forward passes. The anchored representation supplies a continuous source-to-query mean-velocity target for smooth teacher paths.

\begin{figure}[t]
\centering
\includegraphics[width=\columnwidth]{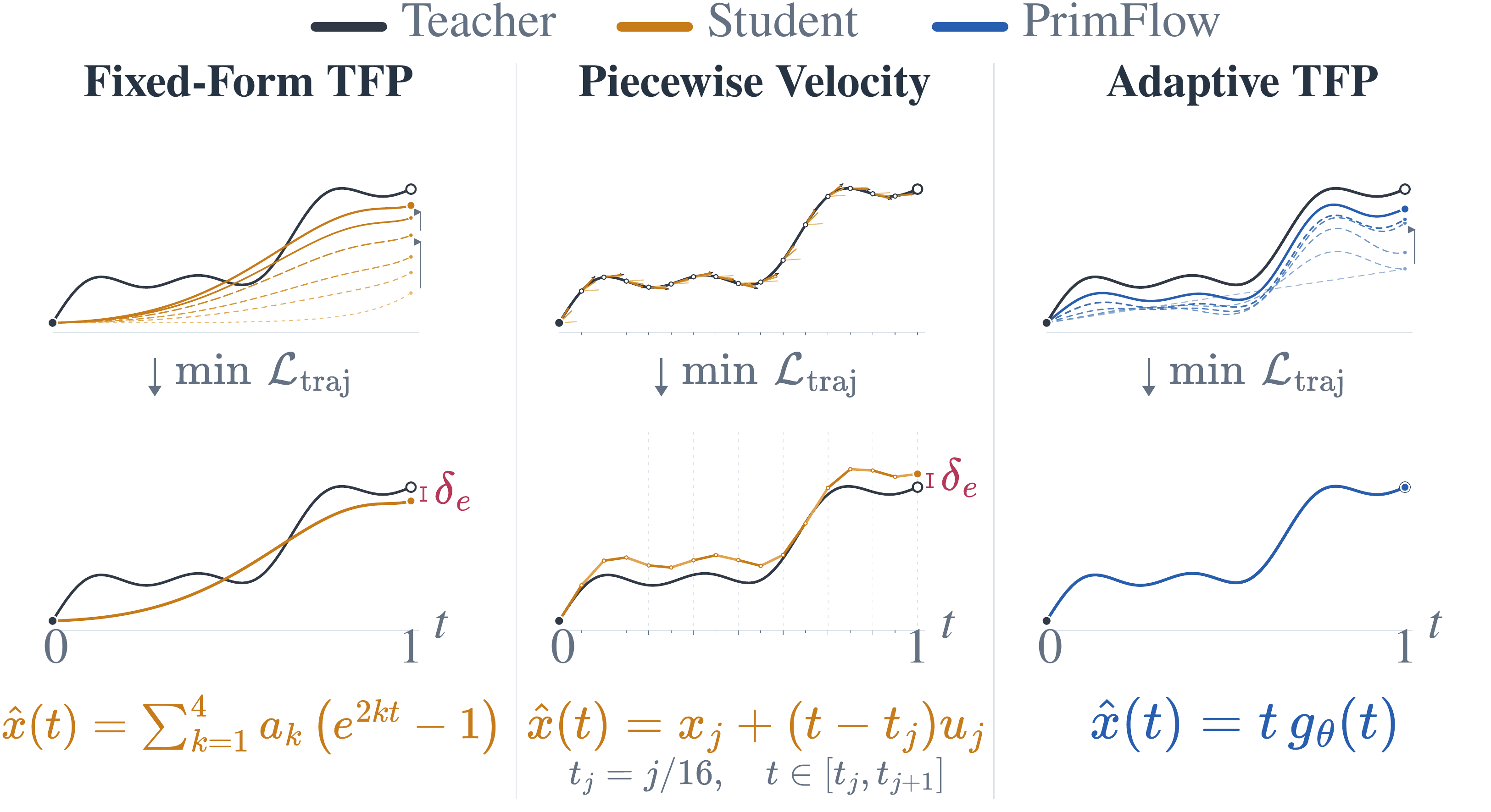}
\caption{\textbf{Geometric mismatch in numerical trajectory fitting.} Top: optimization snapshots or sampled teacher tangents. Bottom: converged trajectories for the simulated path. The displayed exponential and piecewise-linear TFPs retain residual error, whereas the adaptive temporal function closely matches the path.}
\label{fig:geometric-mismatch}
\end{figure}

Fig.~\ref{fig:geometric-mismatch} illustrates the effect on a simulated path. We evaluate the full method against established TFP baselines. On Qwen-Image-20B~\cite{DBLP:journals/corr/abs-2508-02324} and FLUX.2-Klein-base-9B, PrimFlow achieves better overall performance across multiple evaluation metrics at 2/4 network function evaluations (NFE).

Our main contributions are:
\begin{enumerate}
\item \textbf{Adaptive temporal function parameterization.} We construct trajectories through a velocity antiderivative, replacing predefined temporal structures with a lightweight conditional network to relax geometric representation restrictions. We analyze structural mismatch in prescribed temporal families and derive continuous temporal targets for smooth paths from the source.
\item \textbf{Optional deployment with the native prediction structure.} Progressive residual distillation and an endpoint-constrained strategy transfer learned interval updates to the native prediction head. Under a fixed sampling schedule, these strategies allow removal of the auxiliary temporal network and simplify the inference computation graph.
\item \textbf{Large-scale text-to-image evaluation.} Systematic experiments on FLUX.2-Klein-base-9B and Qwen-Image-20B using GenEval, DPG-Bench, and HPSv2 show improved few-step trajectory-distillation quality across most evaluation settings.
\end{enumerate}

\section{Related Work}
\noindent\textbf{Few-step generative models.}
Early few-step generation methods primarily relied on teacher distillation. Progressive Distillation~\cite{DBLP:conf/iclr/SalimansH22} merges multiple teacher sampling steps, while consistency distillation~\cite{DBLP:conf/icml/SongD0S23} constrains consistency between states on the same teacher trajectory. Although these methods use teacher-generated samples, their supervision primarily targets coarse state transitions or endpoints, with limited exploration of intermediate trajectory dynamics. MeanFlow~\cite{DBLP:conf/nips/GengDBKH25} and Flow Map Matching (FMM)~\cite{DBLP:journals/tmlr/BoffiAV25,DBLP:conf/nips/BoffiAV25} further explore interval-based modeling through mean velocities and flow maps, enabling more direct supervision of continuous-time dynamics. These interval formulations require two-time conditioning and, for MeanFlow, JVP-based temporal differentiation, complicating adaptation to pretrained single-time backbones. Another line of work distills distributions: rather than matching individual trajectories, it aligns the teacher's and student's marginal distributions. Representative methods include GAN-based ADD~\cite{DBLP:conf/eccv/SauerLBR24}, LADD~\cite{DBLP:conf/siggrapha/SauerBDBER24}, and APT~\cite{DBLP:conf/icml/Lin0RY0J25}, as well as VSD-based DMD~\cite{DBLP:conf/cvpr/YinG0SDFP24}, $f$-distill~\cite{DBLP:journals/corr/abs-2502-15681}, and SiD~\cite{DBLP:conf/icml/ZhouZWYH24}. These methods can produce high-quality samples but typically involve large auxiliary networks and complex objectives, presenting challenges for training efficiency and stability~\cite{DBLP:conf/nips/YinGP0SDF24}.

\smallskip
\noindent\textbf{TFP-based trajectory distillation.}
Recent trajectory distillation methods design various trajectory constructions to learn from the teacher's dynamics. We group these mechanisms under TFP: the student predicts input-dependent parameters, which are expanded over time through predefined functions or temporal structures. For example, $\pi$-Flow~\cite{DBLP:conf/iclr/ChenZTGWB26} uses a lightweight policy head and numerical integration; ArcFlow~\cite{DBLP:journals/corr/abs-2602-09014} uses continuous momentum parameterization and analytic integration; PDD~\cite{DBLP:journals/corr/abs-2607-26004} predicts interval-wise mean velocities through layer fusion, approximating interval integrals with discrete instantaneous teacher velocities. By keeping temporal expansion outside the backbone, TFP methods reduce such adaptation overhead and help scale trajectory distillation to large pretrained models and industrial deployment. Despite these different designs, learnable parameters do not ensure sufficient expressiveness of the underlying function family. Optimization remains confined to that family, whose trajectory space may exclude the teacher path and thus introduce geometric mismatch and structural approximation error. PTD~\cite{DBLP:journals/corr/abs-2610-11498} learns polynomial trajectory segments, while \href{https://arxiv.org/abs/2610.18472}{SMD/Pivot} uses a training-only residual proxy. The appendix analyzes representative temporal constructions. PrimFlow learns the temporal expansion through a query-conditioned readout rather than selecting a fixed expansion in advance.

\section{Method}
\label{sec:method}
We first analyze the geometric mismatch induced by fixed temporal function families. Motivated by this limitation, we introduce an adaptive temporal parameterization whose training and deployment are summarized in Figure~\ref{fig:primflow-overview}.

\begin{figure*}[t]
    \centering
    \includegraphics[width=\textwidth]{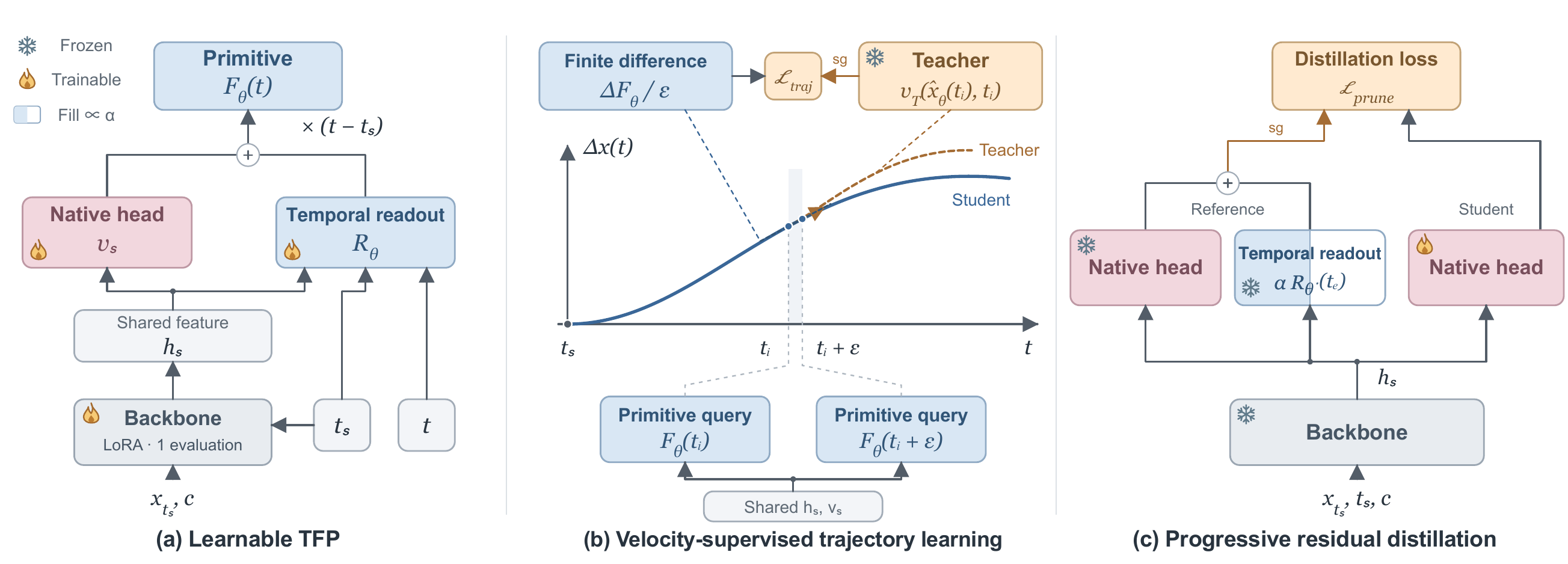}
    \caption{\textbf{PrimFlow overview.} A learnable temporal primitive receives finite-difference velocity supervision, and direct endpoint evaluation performs sampling. Optional progressive residual distillation fits the interval update with the native head for auxiliary-free inference.}
    \label{fig:primflow-overview}
\end{figure*}

\subsection{Preliminaries}
\label{sec:preliminaries}
Flow models commonly construct training supervision through linear interpolation. Given a noise--data pair $(x_0,x_1)$, define
\begin{equation}
\xi_t=(1-t)x_0+tx_1,\qquad u_t=x_1-x_0.
\end{equation}
The model regresses velocity $u_t$ from the current state, time, and condition $c$. Because the input does not specify the complete endpoint pair, the ideal velocity field under squared-error regression is
\begin{equation}
v^*(x,t,c)=\mathbb E[x_1-x_0\mid\xi_t=x,t,c],
\end{equation}
this conditional average varies with state and time. Linear interpolation supervision therefore does not guarantee straight generation trajectories. After training, the teacher trajectory is determined by its learned velocity field:
\begin{equation}
\frac{dx_T(t)}{dt}=v_T(x_T(t),t,c).
\end{equation}
Its shape cannot be prescribed by the training interpolation and must instead be obtained by integrating the velocity field, as described in Flow Matching and Rectified Flow~\cite{DBLP:conf/iclr/LipmanCBNL23,DBLP:conf/iclr/LiuG023}.

Consider a distillation interval $[t_s,t_e]$ with $t_s<t_e$; we omit condition $c$ from the velocity field below. Given a source state $x_{t_s}$, trajectory distillation seeks to approximate the teacher's accumulated interval update using a single student backbone evaluation~\cite{DBLP:conf/nips/SabourFK25,DBLP:journals/corr/abs-2602-09014}:
\begin{equation}
x_T(t_e)-x_{t_s}=\int_{t_s}^{t_e}v_T(x_T(t),t)\,dt.
\end{equation}

\subsection{Geometric Mismatch in Trajectory Distillation}
\label{sec:geometric-mismatch}
\paragraph{Fixed temporal function parameterization.}
TFP-based trajectory distillation typically first predicts a set of parameters through the student backbone~\cite{DBLP:journals/corr/abs-2602-09014}:
\begin{equation}
\mathbf a=G_\varphi(x_{t_s},t_s,c).
\end{equation}
A predefined temporal function $\mathcal B$ then constructs a continuous trajectory:
\begin{equation}
\begin{aligned}\hat x_{\mathbf a}(t)&=x_{t_s}+\mathcal B(t;t_s,\mathbf a),\\\mathcal B(t_s;t_s,\mathbf a)&=0.\end{aligned}
\end{equation}
The student's instantaneous velocity is obtained by differentiating the trajectory with respect to query time:
\begin{equation}
\hat v_{\mathbf a}(t)=\frac{d\hat x_{\mathbf a}(t)}{dt}.
\end{equation}
All query times share the same predicted parameters, so one backbone evaluation supports multiple trajectory and velocity queries within the interval. Teacher velocity supervision at these times jointly optimizes the predicted parameters. The appendix describes the TFPs underlying common trajectory distillation methods.

For a fixed source state and condition, $\mathcal B$ defines the family of trajectories the student can represent:
\begin{equation}
\mathcal X_{\mathcal B}=\{\hat x_{\mathbf a}(\cdot):\mathbf a\in\mathcal A\},
\end{equation}
where $\mathcal A$ is the admissible parameter set. Although the backbone adjusts its predictions for different inputs, the temporal function retains a fixed form. When this family cannot sufficiently approximate the teacher trajectory, parameter optimization alone leaves a structural approximation error. We call this limitation \emph{geometric mismatch}.

\paragraph{Supervision error induced by geometric mismatch.}
Consider velocity supervision along the student's own trajectory. Let $\rho(t)$ denote the time-sampling density within the interval, with $\rho(t)\geq\rho_{\min}>0$. Define the population supervision risk:
\begin{equation}
\mathcal R(\mathbf a)=\mathbb E_{t\sim\rho}\!\left[\left\|\frac{d\hat x_{\mathbf a}(t)}{dt}-v_T(\hat x_{\mathbf a}(t),t)\right\|_2^2\right].
\end{equation}
Assume that the teacher velocity field is Lipschitz continuous in the state with constant $L$, and that student and teacher trajectories share the same source state. Let $\Delta t=t_e-t_s$, and define the uniform-norm distance from the teacher trajectory to the student family:
\begin{equation}
d_{\mathcal B}=\inf_{\hat x\in\mathcal X_{\mathcal B}}\|\hat x-x_T\|_\infty.
\end{equation}
Gr\"onwall's inequality gives
\begin{equation}
\inf_{\mathbf a\in\mathcal A}\mathcal R(\mathbf a)\geq\frac{\rho_{\min}e^{-2L\Delta t}}{\Delta t}\,d_{\mathcal B}^{\,2}.
\end{equation}
When $d_{\mathcal B}>0$, the velocity-supervision risk has a strictly positive lower bound. This includes the case where a closed student family excludes the teacher trajectory. The bound concerns the temporal family after optimizing all its admissible parameters, including any learnable rates.

This analysis motivates changing how the temporal function is formed. PrimFlow keeps the source features fixed and passes query time to the readout. Local supervision and endpoint prediction share one learned displacement.

\subsection{Learnable Antiderivative Parameterization}
\label{sec:antiderivative}
\paragraph{Conditional temporal function.}
We condition the temporal function on the final hidden features $h_s$ from a single backbone evaluation at the source:
\begin{equation}
h_s=G_{\mathrm{pre},\psi}(x_{t_s},t_s,c).
\end{equation}
The source features $h_s$ stay fixed across the interval; changing the query time $t$ changes the readout while preserving its source conditioning.

\paragraph{Trajectory construction from velocity antiderivatives.}
The auxiliary network learns the displacement function $F_\theta$. The teacher's instantaneous velocities supervise its query-time derivative (the student velocity):
\begin{equation}
\hat v_\theta(t)=\frac{\partial F_\theta(t;h_s,t_s)}{\partial t}.
\end{equation}
By the fundamental theorem of calculus, the student trajectory is
\begin{equation}
\begin{aligned}\hat x_\theta(t)=x_{t_s}&+F_\theta(t;h_s,t_s)\\&-F_\theta(t_s;h_s,t_s).\end{aligned}
\end{equation}
Because inference always starts from the source, we anchor the antiderivative there by imposing $F_\theta(t_s;h_s,t_s)=0$. Specifically, we parameterize it as
\begin{equation}
F_\theta(t;h_s,t_s)=(t-t_s)g_\theta(h_s,t_s,t).
\end{equation}
The student trajectory then becomes
\begin{equation}
\hat x_\theta(t)=x_{t_s}+(t-t_s)g_\theta(h_s,t_s,t).
\end{equation}
This construction guarantees $\hat x_\theta(t_s)=x_{t_s}$ and avoids an additional antiderivative evaluation at the source. Once backbone features are available, predicting the target state of each sampling interval requires only one auxiliary-network forward pass. Local velocity supervision and direct endpoint prediction share a trajectory representation without numerical integration of student velocities.

The corresponding instantaneous velocity satisfies~\cite{DBLP:conf/nips/GengDBKH25}
\begin{equation}
\begin{aligned}\hat v_\theta(t)={}&g_\theta(h_s,t_s,t)\\&+(t-t_s)\frac{\partial g_\theta(h_s,t_s,t)}{\partial t}.\end{aligned}
\end{equation}
When differentiating with respect to query time, the source state, source time, and backbone features remain fixed; only the auxiliary network depends on query time.

\paragraph{Continuous temporal target.}
For any continuously differentiable teacher trajectory satisfying $x_T(t_s)=x_{t_s}$, define
\begin{equation}
g_T(t)=\begin{cases}\displaystyle\frac{x_T(t)-x_{t_s}}{t-t_s},&t>t_s,\\[6pt]\dot x_T(t_s),&t=t_s.\end{cases}
\end{equation}
Then $g_T$ is continuous at the source and exactly satisfies
\begin{equation}
x_T(t)=x_{t_s}+(t-t_s)g_T(t).
\end{equation}
For $t>t_s$, $g_T(t)$ is the teacher's source-to-query mean velocity. The identity specifies a continuous target for the readout, with no prescribed curvature or rates.

\subsection{Lightweight Residual Training}
\label{sec:trajectory-training}
\paragraph{Lightweight residual architecture.}
As in other trajectory distillation methods~\cite{DBLP:journals/corr/abs-2602-09014}, we fine-tune the pretrained DiT~\cite{DBLP:conf/iccv/PeeblesX23} backbone with LoRA~\cite{DBLP:conf/iclr/HuSWALWWC22} while training its native velocity prediction head. The source-state features are
\begin{equation}
h_s=G_{\mathrm{pre},\psi}(x_{t_s},t_s,c).
\end{equation}
The native velocity head supplies a base prediction:
\begin{equation}
v_s=H_{\mathrm{nat},\lambda}(h_s).
\end{equation}
A lightweight auxiliary network conditions on the query time $t$ and uses the shared source features $h_s$ to predict a residual $R_\theta(h_s,t_s,t)$. Adding this residual to the native prediction $v_s$ gives the source-to-query mean velocity:
\begin{equation}
g_\theta(h_s,t_s,t)=v_s+R_\theta(h_s,t_s,t).
\end{equation}
Multiplying this sum by $t-t_s$ gives the source-anchored antiderivative at any query time $t\in[t_s,t_e]$. The resulting trajectory is
\begin{equation}
\hat x_\theta(t)=x_{t_s}+(t-t_s)[v_s+R_\theta(h_s,t_s,t)].
\end{equation}
The displacement thus adds a query-dependent correction to the linear update $(t-t_s)v_s$, while retaining the source anchor. The backbone, native head, and temporal readout are trained jointly.

\paragraph{Finite-difference time derivatives.}
During training, we approximate the antiderivative's time derivative using auxiliary-network outputs at neighboring query times. For a small offset $\epsilon>0$, choose $t$ such that both $t$ and $t+\epsilon$ belong to $[t_s,t_e]$, and define
\begin{equation}
\hat v_{\theta,\epsilon}(t)=\frac{F_\theta(t+\epsilon;h_s,t_s)-F_\theta(t;h_s,t_s)}{\epsilon}.
\end{equation}
Under the residual parameterization, this becomes
\begin{equation}
\begin{aligned}\hat v_{\theta,\epsilon}(t)={}&v_s+\frac{1}{\epsilon}\Big[(t+\epsilon-t_s)R_\theta(h_s,t_s,t+\epsilon)\\&\qquad\qquad-(t-t_s)R_\theta(h_s,t_s,t)\Big].\end{aligned}
\end{equation}
The quotient is the student's mean velocity over $[t,t+\epsilon]$ and approaches its instantaneous velocity as $\epsilon\to 0$. Both evaluations reuse $h_s$ and $v_s$, so the time difference involves only the temporal readout.

\paragraph{Trajectory supervision.}
For each source state, we sample $m$ supervision times $t_i$ within the interval. Training initially mixes teacher and student motion; after this curriculum, we query the teacher along the student trajectory~\cite{DBLP:conf/iclr/ChenZTGWB26,DBLP:journals/corr/abs-2602-09014}:
\begin{equation}
u_i=\operatorname{sg}\!\left[v_T(\hat x_\theta(t_i),t_i)\right],
\end{equation}
where $\operatorname{sg}$ denotes stop-gradient. We minimize
\begin{equation}
\mathcal L_{\mathrm{traj}}(\psi,\lambda,\theta)=\mathbb E\!\left[\frac{1}{m}\sum_{i=1}^{m}\left\|\hat v_{\theta,\epsilon}(t_i)-u_i\right\|_2^2\right].
\end{equation}
The teacher remains frozen and supplies supervision targets. Gradients through the student's finite-difference velocity branch update the auxiliary network and model backbone. Local supervision at different times jointly constrains the same continuous temporal function. After training, the target state is obtained directly:
\begin{equation}
\hat x_\theta(t_e)=x_{t_s}+(t_e-t_s)\left[v_s+R_\theta(h_s,t_s,t_e)\right].
\end{equation}

\begin{table*}[!t]
\centering
\caption{\textbf{Few-step text-to-image generation.} NFE counts backbone evaluations; $50\times2$ denotes 50 teacher steps with conditional and unconditional evaluations for guidance. HPSv2.1 and teacher-alignment FID~\cite{DBLP:conf/nips/HeuselRUNH17} (tFID) are evaluated on the HPSv2 benchmark~\cite{DBLP:journals/corr/abs-2306-09341}; tFID is computed against the corresponding 50-step teacher. Higher is better for GenEval~\cite{DBLP:conf/nips/GhoshHS23}, DPG-Bench~\cite{DBLP:journals/corr/abs-2403-05135}, and HPSv2.1; lower is better for tFID. PrimFlow denotes the standard model retaining the auxiliary network. Best and second-best results are marked in bold and underlined within each backbone/NFE group. A dash denotes unavailable tFID.}
\label{tab:quantitative}
\small
\setlength{\tabcolsep}{7.0pt}
\begin{tabular*}{\textwidth}{@{\extracolsep{\fill}}lccccc@{}}
\toprule
\multirow{2}{*}{Model} & \multirow{2}{*}{NFE$\downarrow$} & \multirow{2}{*}{GenEval$\uparrow$} & \multirow{2}{*}{DPG-Bench$\uparrow$} & \multicolumn{2}{c}{HPSv2} \\
\cmidrule(lr){5-6}
& & & & HPSv2.1$\uparrow$ & tFID$\downarrow$ \\
\midrule
FLUX.2-Klein-base-9B & $50\!\times\!2$ & 0.84 & 85.48 & 28.81 & -- \\
\midrule
PDD (Midpoint)~\cite{DBLP:journals/corr/abs-2607-26004} & 2 & 0.77 & \underline{85.62} & \underline{28.78} & 24.56 \\
ArcFlow~\cite{DBLP:journals/corr/abs-2602-09014} & 2 & \underline{0.79} & 85.54 & 28.15 & \underline{20.08} \\
\method & 2 & \textbf{0.82} & \textbf{85.90} & \textbf{29.07} & \textbf{18.78} \\
\midrule
PDD (Midpoint) & 4 & 0.75 & \textbf{85.64} & \underline{28.90} & 25.10 \\
ArcFlow & 4 & \underline{0.83} & 85.41 & 28.72 & \underline{17.37} \\
\method & 4 & \textbf{0.84} & \underline{85.57} & \textbf{29.06} & \textbf{16.34} \\
\midrule
Qwen-Image-20B~\cite{DBLP:journals/corr/abs-2508-02324} & $50\!\times\!2$ & 0.86 & 88.30 & 30.83 & -- \\
\midrule
TwinFlow~\cite{DBLP:conf/iclr/ChengSLL26} & 2 & 0.82 & 86.67 & 29.86 & 22.13 \\
PDD (Euler)~\cite{DBLP:journals/corr/abs-2607-26004} & 2 & \underline{0.86} & 88.04 & 29.59 & 17.83 \\
PDD (Midpoint) & 2 & \underline{0.86} & 88.10 & \underline{30.15} & 17.41 \\
ArcFlow~\cite{DBLP:journals/corr/abs-2602-09014} & 2 & 0.85 & \textbf{88.46} & 30.04 & \underline{16.91} \\
\method & 2 & \textbf{0.87} & \underline{88.24} & \textbf{30.28} & \textbf{15.82} \\
\midrule
TwinFlow & 4 & 0.82 & 86.18 & 30.01 & 21.05 \\
$\pi$-Flow~\cite{DBLP:conf/iclr/ChenZTGWB26} & 4 & 0.85 & 88.11 & 30.94 & \underline{14.58} \\
PDD (Euler) & 4 & \underline{0.86} & 88.45 & 31.05 & 15.91 \\
PDD (Midpoint) & 4 & \underline{0.86} & \underline{88.66} & \textbf{31.33} & 15.55 \\
ArcFlow & 4 & 0.85 & 87.70 & 30.67 & 14.65 \\
\method & 4 & \textbf{0.87} & \textbf{88.75} & \underline{31.30} & \textbf{14.37} \\
\bottomrule
\end{tabular*}
\end{table*}

\subsection{Deployment without Auxiliary Networks}
\label{sec:deployment}
We provide optional strategies that transfer learned interval updates to the native prediction head, removing the auxiliary temporal network to simplify the deployment computation graph. Under a fixed few-step sampling schedule, each source time has a unique target time, allowing the native head to learn the corresponding interval update from the source state and its time condition.

\paragraph{Progressive residual distillation.}
We use the trajectory-trained model as the teacher for this stage. Let $\theta^\star$ denote its auxiliary-network parameters and $v_s^\star$ its velocity-head prediction. Only the student velocity-head parameters $\lambda$ are trainable; all other parameters are frozen. At target time $t_e$, construct the progressive supervision target:
\begin{equation}
v_{\mathrm{tea}}=v_s^\star+\alpha R_{\theta^\star}(h_s,t_s,t_e).
\end{equation}
The residual weight $\alpha\in[0,1]$ grows smoothly during training, increasing the auxiliary residual's contribution to the target. We update the student output head with
\begin{equation}
\mathcal L_{\mathrm{prune}}(\lambda)=\mathbb E\!\left[\left\|H_{\mathrm{nat},\lambda}(h_s)-\operatorname{sg}[v_{\mathrm{tea}}]\right\|_2^2\right].
\end{equation}
This process progressively transfers the teacher's full interval update to the student output head. After training, we remove the temporal readout and compute
\begin{equation}
x_{t_e}=x_{t_s}+(t_e-t_s)H_{\mathrm{nat},\lambda^\star}(h_s).
\end{equation}
The deployed model retains the native prediction head's structure and interface, with its output representing the mean velocity over the specified interval. Each sampling interval requires one backbone evaluation.

\paragraph{Endpoint-constrained variant.}
Alternatively, we constrain the auxiliary residual to vanish at the target endpoint during training. Define the temporal weight
\begin{equation}
\alpha(t)=\frac{t-t_e}{t_s-t_e},
\end{equation}
and modify the conditional temporal function as
\begin{equation}
\begin{aligned}g_{\lambda,\theta}^{\mathrm{br}}(h_s,t_s,t)={}&H_{\mathrm{nat},\lambda}(h_s)\\&+\alpha(t)R_\theta(h_s,t_s,t).\end{aligned}
\end{equation}
The resulting student trajectory is
\begin{equation}
\hat x^{\mathrm{br}}_{\lambda,\theta}(t)=x_{t_s}+(t-t_s)g^{\mathrm{br}}_{\lambda,\theta}(h_s,t_s,t).
\end{equation}
This variant follows the trajectory-training procedure in Section~\ref{sec:trajectory-training}. Because $\alpha(t_e)=0$, endpoint prediction depends only on the native head:
\begin{equation}
\hat x^{\mathrm{br}}_{\lambda,\theta}(t_e)=x_{t_s}+(t_e-t_s)H_{\mathrm{nat},\lambda}(h_s).
\end{equation}
The auxiliary network fits the intermediate trajectory, while the native head directly learns the endpoint update. The auxiliary network can therefore be removed after training without an additional residual-distillation stage.

\begin{figure*}[!t]
    \centering
    \includegraphics[width=\textwidth]{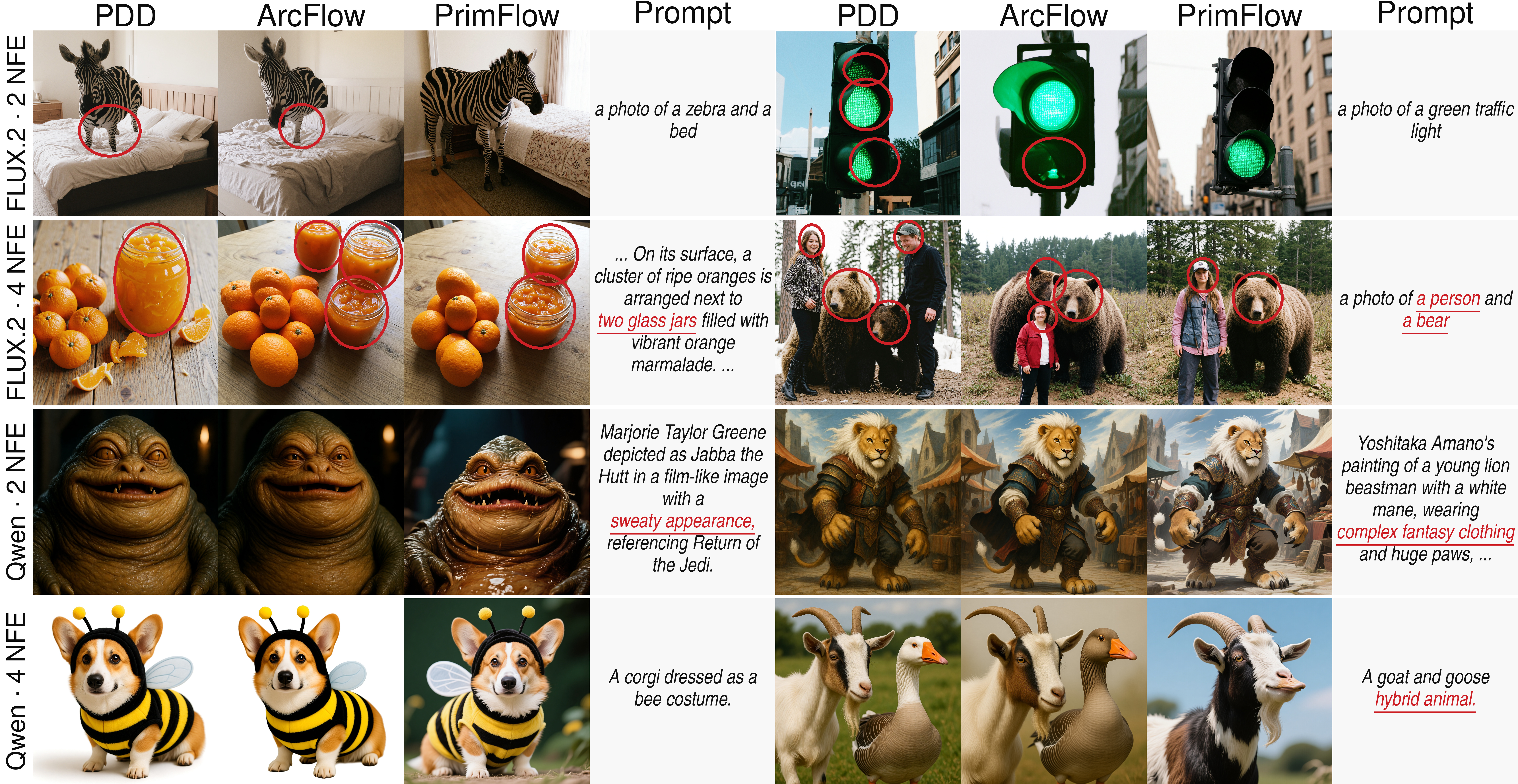}
    \caption{\textbf{Qualitative comparison.} PrimFlow improves prompt fidelity and visual coherence. Red underlines mark key prompt terms; red circles identify counted objects and malformed regions.}
    \label{fig:qualitative-comparison}
\end{figure*}

\section{Experiments}
\label{sec:experiments}
We compare adaptive-TFP trajectory distillation with conventional TFP methods on FLUX.2-Klein-base-9B and Qwen-Image-20B to evaluate the temporal modeling design. ArcFlow serves as a representative fixed-form baseline. We further ablate readout capacity and finite-difference span, using FLUX.2 at 2 NFE unless otherwise stated.

\subsection{Implementation Details}
\label{sec:implementation-details}
We train on 32 NVIDIA A100 GPUs using a 200K-prompt subset of \href{https://huggingface.co/datasets/Lakonik/t2i-prompts-3m}{T2I-Prompts-3M}, following $\pi$-Flow~\cite{DBLP:conf/iclr/ChenZTGWB26}. We evaluate 2/4-NFE generation at $1024\times1024$ resolution. PrimFlow uses a LoRA rank and alpha of 256 for training and inference. On the same hardware, we match GPU time with PDD and both GPU time and training trajectory count with ArcFlow. The appendix reports the architecture and training configurations. The readout adds approximately 0.29\% / 0.13\% of the backbone parameters, comparable to PDD's extra output-head parameters.

\subsection{Quantitative and Qualitative Results}
\label{sec:quantitative-results}
\paragraph{Quantitative results.}
Across the four backbone/NFE configurations and four metrics, PrimFlow ranks first in 13 of 16 comparisons and second in the remaining three (Table~\ref{tab:quantitative}). It leads all metrics on FLUX.2 at 2 NFE; on FLUX.2 at 4 NFE and Qwen at 2 NFE, it leads GenEval, HPSv2.1, and tFID while ranking second on DPG-Bench. On Qwen at 4 NFE, it leads GenEval, DPG-Bench, and tFID while ranking second on HPSv2.1. Thus, PrimFlow achieves the best GenEval and tFID in every configuration while remaining top-two on both DPG-Bench and HPSv2.1.

\paragraph{Qualitative results.}
Figure~\ref{fig:qualitative-comparison} illustrates differences in object structure and composition. PrimFlow renders the zebra with visibly separated legs and generates the requested two marmalade jars, compared with one in PDD and three in ArcFlow. On Qwen at 4 NFE, PrimFlow combines the horns and bill in one goat--goose hybrid, while both baselines depict the animals side by side. These examples indicate gains in local structure and compositional fidelity, including object separation, counting, and attribute binding.

\subsection{Ablation Studies}
\label{sec:ablations}
\paragraph{Trajectory fitting comparison.}
Generation trajectories lie in a high-dimensional latent space, and finite time samples capture only part of the continuous teacher path. On FLUX.2 at 2 NFE, we use the logged local velocity regression loss to compare trajectory fitting indirectly. ArcFlow provides the most direct fixed-form TFP comparison because both methods construct continuous trajectories and obtain interval displacements from their temporal functions. This comparison focuses on the temporal parameterization without introducing PDD's discrete velocity fusion or $\pi$-Flow's numerical policy integration. After the initial drop, PrimFlow maintains a lower moving-average loss for most of training (Figure~\ref{fig:trajectory-training-loss}), consistent with closer agreement with the sampled teacher velocity targets.

\begin{figure}[tbp]
    \centering
    \includegraphics[width=\columnwidth]{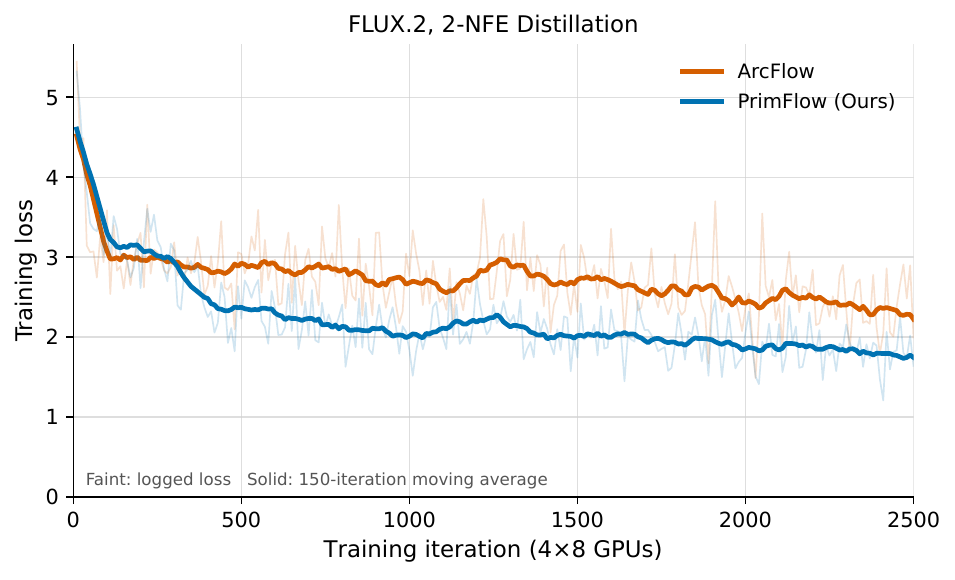}
    \caption{\textbf{Trajectory training loss on FLUX.2 at 2 NFE.} Light curves show the logged losses and solid curves show 150-iteration moving averages.}
    \label{fig:trajectory-training-loss}
\end{figure}

\paragraph{Auxiliary-network capacity.}
We vary the width of the temporal readout at fixed depth. The two compact readouts perform similarly, with no further overall improvement from the largest readout (Table~\ref{tab:capacity-ablation}). The smallest readout retains the generation gains over ArcFlow in this setting (Table~\ref{tab:quantitative}).

\begin{table}[!htb]
\centering
\caption{\textbf{Auxiliary-network width} on FLUX.2 at 2 NFE, with eight blocks and the readout used for endpoint prediction.}
\label{tab:capacity-ablation}
\scriptsize
\setlength{\tabcolsep}{2.7pt}
\begin{tabular*}{\columnwidth}{@{\extracolsep{\fill}}lccccc@{}}
\toprule
Width & Params & GenEval$\uparrow$ & DPG$\uparrow$ & HPSv2.1$\uparrow$ & tFID$\downarrow$ \\
\midrule
1024 & 98.25M & 0.81 & 85.75 & 28.70 & \textbf{18.49} \\
512 & 25.80M & \textbf{0.82} & \underline{85.90} & \textbf{29.07} & 18.78 \\
256 & 7.07M & \textbf{0.82} & \textbf{86.21} & \underline{28.87} & \underline{18.72} \\
\bottomrule
\end{tabular*}
\end{table}

\paragraph{Supervision-window length.}
Table~\ref{tab:fd-span-ablation} compares finite-difference spans on FLUX.2 at 2 NFE to assess whether coarser supervision windows degrade performance. Reducing the finite-difference span does not improve overall performance, indicating that the current configuration already supports effective trajectory distillation.

\begin{table}[!htb]
\centering
\caption{\textbf{Finite-difference span} on FLUX.2 at 2 NFE, using a width-512, depth-8 readout.}
\label{tab:fd-span-ablation}
\scriptsize
\setlength{\tabcolsep}{2.8pt}
\begin{tabular*}{\columnwidth}{@{\extracolsep{\fill}}lcccc@{}}
\toprule
Supervision & GenEval$\uparrow$ & DPG$\uparrow$ & HPSv2.1$\uparrow$ & tFID$\downarrow$ \\
\midrule
$3/128$ & \textbf{0.82} & 85.70 & 28.59 & \textbf{18.48} \\
$1/128$ & \textbf{0.82} & \underline{85.90} & \textbf{29.07} & 18.78 \\
$1/256$ & 0.81 & \textbf{85.97} & \underline{28.68} & \underline{18.54} \\
\bottomrule
\end{tabular*}
\end{table}

\begin{table}[!htb]
\centering
\caption{\textbf{Endpoint-constrained deployment} on FLUX.2 at 2 NFE. The constrained model uses only the native head. Both variants use width 512, depth 8, span $1/128$, seed 42, and matched training sample budgets.}
\label{tab:endpoint-constrained-ablation}
\scriptsize
\setlength{\tabcolsep}{2.6pt}
\begin{tabular*}{\columnwidth}{@{\extracolsep{\fill}}llcccc@{}}
\toprule
Variant & Inference & GenEval$\uparrow$ & DPG$\uparrow$ & HPSv2.1$\uparrow$ & tFID$\downarrow$ \\
\midrule
Standard & Head + readout & \textbf{0.82} & 85.90 & \textbf{29.07} & 18.78 \\
Constrained & Head only & 0.81 & \textbf{86.00} & 28.38 & \textbf{18.72} \\
\bottomrule
\end{tabular*}
\end{table}

\paragraph{Endpoint-constrained deployment.}
Table~\ref{tab:endpoint-constrained-ablation} evaluates the variant that forces the auxiliary residual to vanish at every sampling endpoint during trajectory training. It removes the auxiliary network directly, without progressive residual distillation. The endpoint constraint may introduce unintended optimization restrictions, as reflected in the marked decrease in HPSv2.1. This variant enables auxiliary-free sampling where perceptual preference is secondary.

\begin{figure}[!htb]
    \centering
    \begin{minipage}[t]{0.235\columnwidth}
        \centering
        \includegraphics[width=\linewidth]{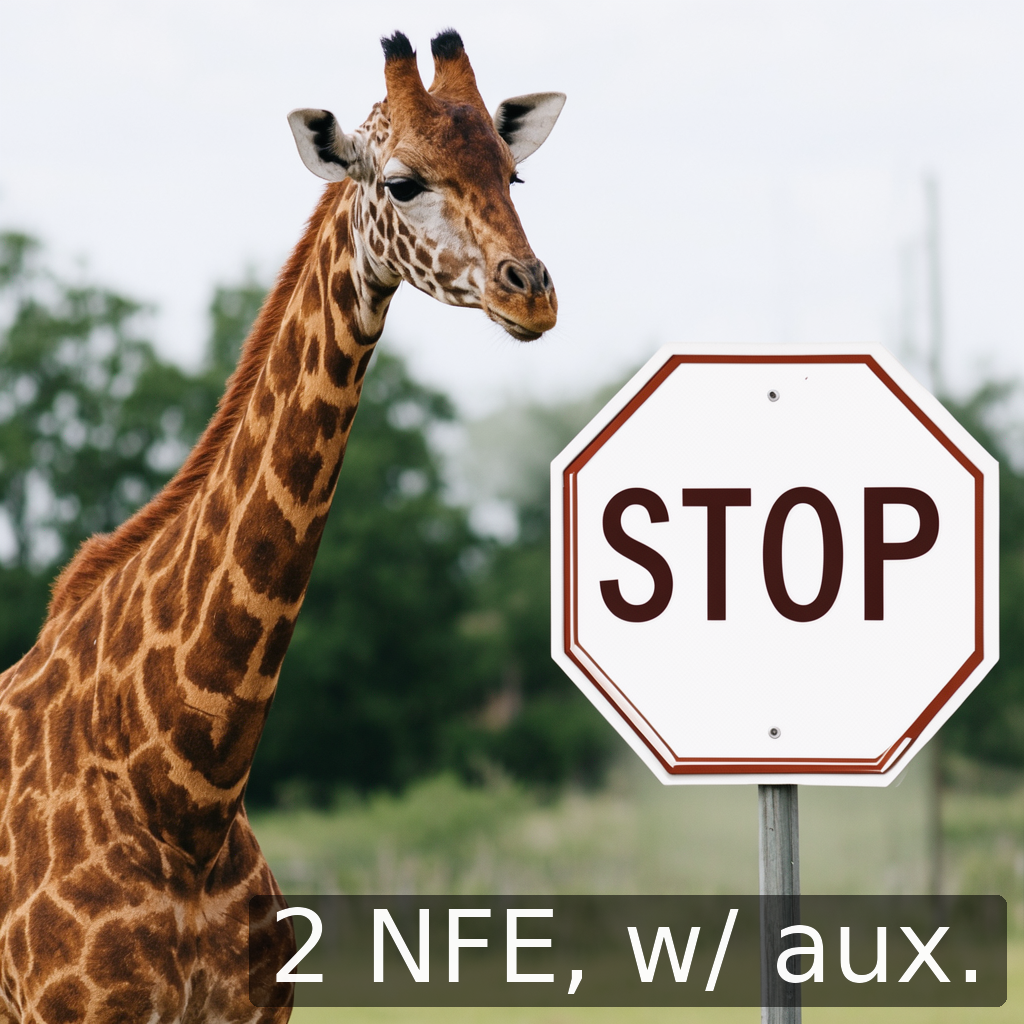}
    \end{minipage}\hfill%
    \begin{minipage}[t]{0.235\columnwidth}
        \centering
        \includegraphics[width=\linewidth]{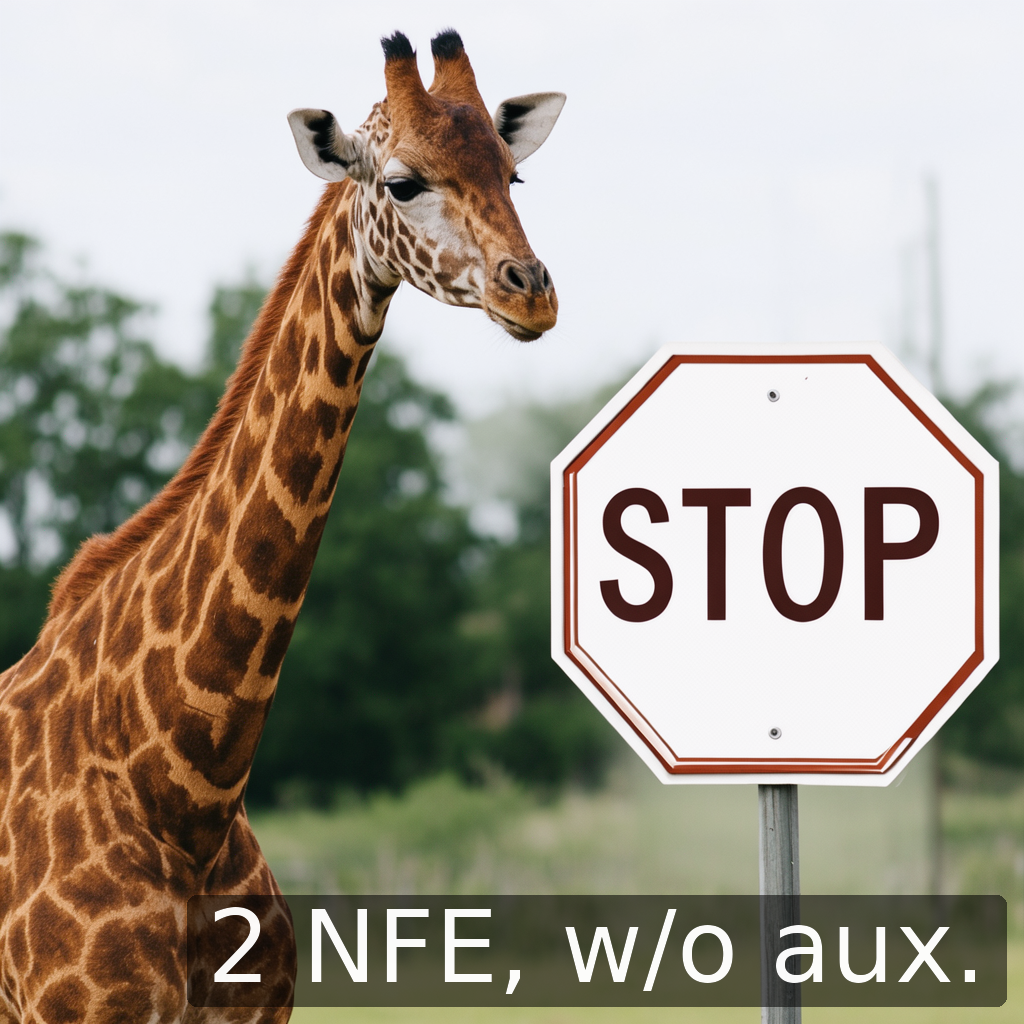}
    \end{minipage}\hfill%
    \begin{minipage}[t]{0.235\columnwidth}
        \centering
        \includegraphics[width=\linewidth]{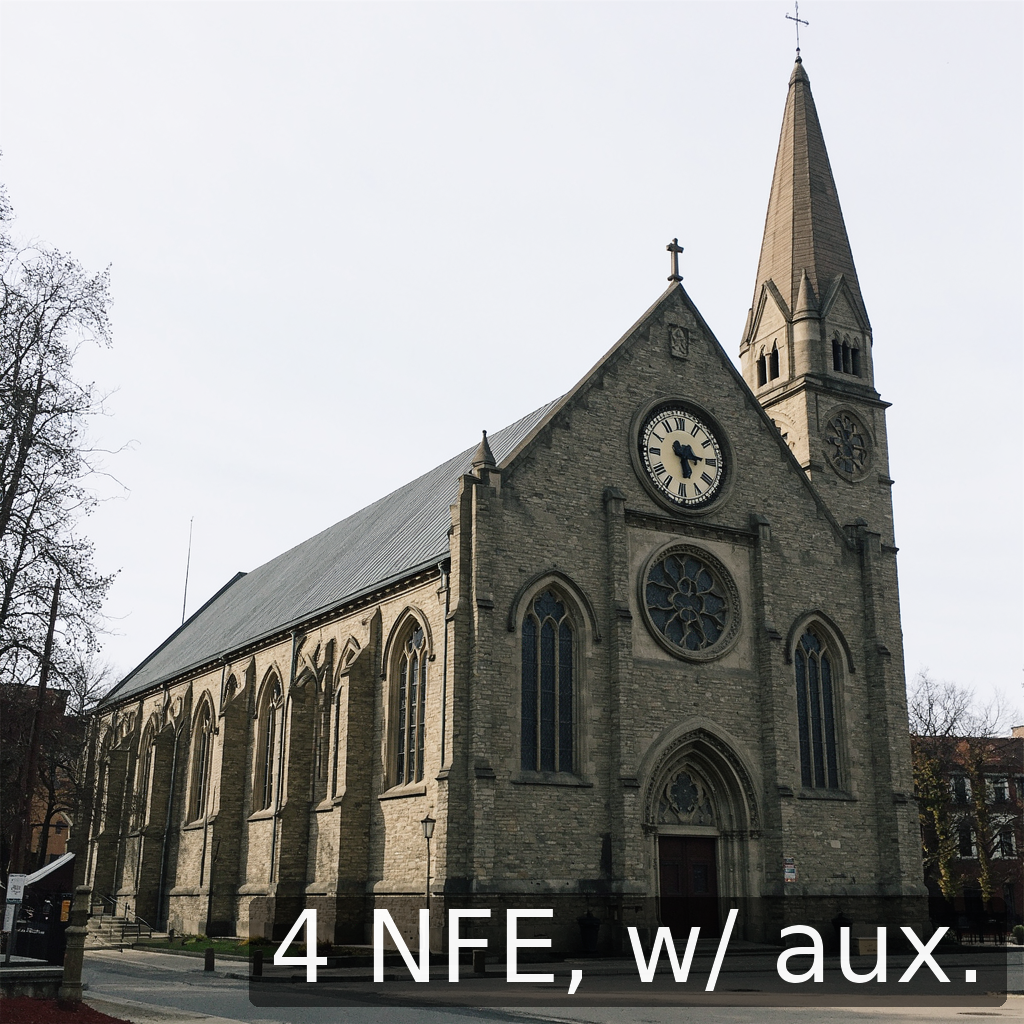}
    \end{minipage}\hfill%
    \begin{minipage}[t]{0.235\columnwidth}
        \centering
        \includegraphics[width=\linewidth]{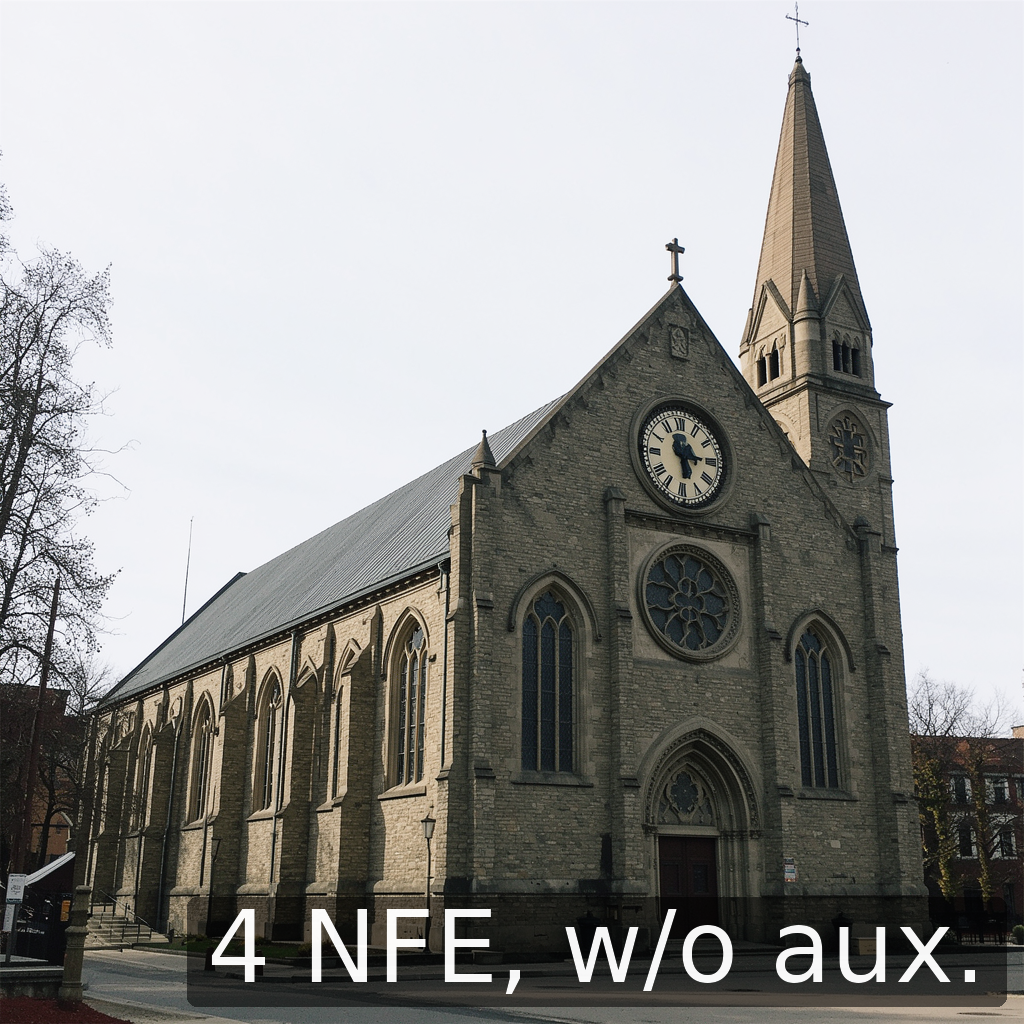}
    \end{minipage}
    \caption{\textbf{Auxiliary-network removal.} Matched FLUX.2 examples with (w/) and without (w/o) the auxiliary network at 2 and 4 NFE, using identical prompts and seeds within each pair.}
    \label{fig:pruning-comparison}
\end{figure}

\begin{table}[!htb]
\centering
\caption{\textbf{Paired-output fidelity after auxiliary-network removal.} The auxiliary-network model is the reference. LPIPS~\cite{DBLP:conf/cvpr/ZhangIESW18} uses AlexNet v0.1 calibrated weights.}
\label{tab:deployment-fidelity}
\small
\setlength{\tabcolsep}{8pt}
\begin{tabular}{lccc}
\toprule
Backbone & NFE & PSNR$\uparrow$ & LPIPS$\downarrow$ \\
\midrule
\multirow{2}{*}{FLUX.2} & 2 & 35.843 & 0.011064 \\
& 4 & 29.028 & 0.031034 \\
\bottomrule
\end{tabular}
\end{table}

\begin{table}[!htb]
\centering
\caption{\textbf{Generation quality before and after auxiliary-network removal.} Paired evaluation on FLUX.2 at 2 and 4 NFE with seed 42. tFID uses the teacher reference.}
\label{tab:pruning-quality}
\small
\setlength{\tabcolsep}{2.6pt}
\begin{tabular*}{\columnwidth}{@{\extracolsep{\fill}}lcccc@{}}
\toprule
NFE / stage & GenEval$\uparrow$ & DPG$\uparrow$ & HPSv2.1$\uparrow$ & tFID$\downarrow$ \\
\midrule
2 / before & 0.82 & 85.90 & 29.07 & 18.78 \\
2 / after & 0.82 & 85.66 & 29.02 & 18.73 \\
\midrule
4 / before & 0.84 & 85.57 & 29.06 & 16.34 \\
4 / after & 0.84 & 85.42 & 29.03 & 16.26 \\
\bottomrule
\end{tabular*}
\end{table}

\paragraph{Deployment after auxiliary-network removal.}
Figure~\ref{fig:pruning-comparison} shows that removing the auxiliary network largely preserves the image layout at both 2 and 4 NFE. Paired outputs agree more closely at 2 NFE according to PSNR and LPIPS (Table~\ref{tab:deployment-fidelity}). Generation quality changes little (Table~\ref{tab:pruning-quality}): GenEval is unchanged, DPG-Bench and HPSv2.1 decrease slightly, and tFID slightly improves at both sampling budgets.

\section{Limitations}
PrimFlow relaxes the representational constraints of fixed-form TFP through an adaptive temporal function, but accurate fitting still depends on the interval information retained by a few backbone evaluations. Because the readout cannot recover dynamics absent from the shared source features, increasing its capacity alone may not improve trajectory accuracy. This information bottleneck is shared by few-step trajectory distillation methods.

\section{Conclusion}
PrimFlow learns an adaptive TFP with a lightweight auxiliary network and improves overall text-to-image performance across both backbones and sampling budgets.On FLUX.2 at 2 NFE, compact readouts retain the overall performance advantage, whereas larger ones provide no overall gain, suggesting that backbone features, rather than readout capacity, are the remaining bottleneck.Shorter finite-difference spans yield no consistent gains, supporting the default span for instantaneous velocity supervision. Two optional strategies remove the auxiliary network for simpler inference, establishing output-side temporal modeling as a practical direction for large pretrained generative models.